 \documentclass[tablecaption=bottom,wcp]{jmlr} 

\usepackage{booktabs}
\usepackage[load-configurations=version-1]{siunitx} 

\theorembodyfont{\upshape}
\theoremheaderfont{\scshape}
\theorempostheader{:}
\theoremsep{\newline}

\jmlrvolume{1}
\jmlryear{2018}
\jmlrsubmitted{submission date}
\jmlrpublished{publication date}
\jmlrworkshop{workshop title} 

\title[Aggregating Strategies for Long-term Forecasting]{Aggregating Strategies for Long-term Forecasting}

 \author{\Name{Alexander Korotin} \Email{a.korotin@skoltech.ru}\\
 \addr Skolkovo Institute of Science and Technology,\\Nobel street, 3, Moscow, Moskovskaya oblast', Russia
 \AND 
 \Name{Vladimir V'yugin\nametag{\thanks{Vladimir V'yugin was supported by the Russian Science Foundation grant (project 14-50-00150).}}} \Email{vyugin@iitp.ru}\\
 \addr Institute for Information Transmission Problems,\\Bolshoy Karetny per. 19, build.1, Moscow, Russia
 \AND
 \Name{Evgeny Burnaev\nametag{\thanks{Evgeny Burnaev was supported by the Ministry of Education and Science of Russian
Federation, grant No. 14.606.21.0004, grant code: RFMEFI60617X0004.}}} \Email{e.burnaev@skoltech.ru}\\
 \addr Skolkovo Institute of Science and Technology,\\Nobel street, 3, Moscow, Moskovskaya oblast', Russia
}

\editor{Editor's name}

\usepackage{amsmath}
\usepackage{bm}
\usepackage{url}
\usepackage{graphicx}
\usepackage{indentfirst}
\usepackage{mathtools}
\usepackage{amssymb}

\DeclareMathOperator*{\argmin}{arg\,min}

\usepackage{hyperref}

\usepackage{amsfonts}

\begin{document}

\maketitle

\begin{abstract}
The article is devoted to investigating the application of aggregating algorithms to the problem of the long-term forecasting. We examine the classic aggregating algorithms based on the exponential reweighing. For the general Vovk's aggregating algorithm we provide its generalization for the long-term forecasting. For the special basic case of Vovk's algorithm we provide its two modifications for the long-term forecasting. The first one is theoretically close to an optimal algorithm and is based on replication of independent copies. It provides the time-independent regret bound with respect to the best expert in the pool. The second one is not optimal but is more practical and has $O(\sqrt{T})$ regret bound, where $T$ is the length of the game.
\end{abstract}

\section{Introduction}

We consider the online game of prediction with experts' advice. A master (aggregating) algorithm at every step $t=1,\dots,T$ of the game has to combine aggregated prediction from predictions of a finite pool of $N$ experts (see e.g. \cite{LiW94}, \cite{FrS97}, \cite{Vov90}, \cite{VoV98}, \cite{cesa-bianchi}, \cite{AdaAda16} among others). We investigate the adversarial case, that is, no assumptions are made about the nature of the data (stochastic, deterministic, etc.).

In the classical online scenario, all predictions at step $t$ are made for the next step $t+1$. The true outcome is revealed immediately at the beginning of the next step of the game and the algorithm suffers loss using a loss function. 

In contrast to the classical scenario, we consider the long-term forecasting. At each step $t$ of the game, the predictions are made for some pre-determined point $t+D$ ahead (where $D\geq 1$ is some fixed known horizon), and the true outcome is revealed only at step $t+D$.

The performance of the aggregating algorithm is measured by the regret over the entire game. The regret $R_T$ is the difference between the cumulative loss of the online aggregating algorithm and the loss of some offline comparator. A typical offline comparator is the best fixed expert in the pool or the best fixed convex linear combination of experts. The goal of an aggregating algorithm is to minimize the regret, that is, $R_{T}\rightarrow \min$. 

It turns out that there exists a wide range of aggregating algorithms for the classic scenario ($D=1$). The majority of them are based on the exponential reweighing methods (see \cite{LiW94}, \cite{FrS97}, \cite{Vov90}, \cite{VoV98}, \cite{cesa-bianchi}, \cite{AdaAda16}, etc.). At the same time, several algorithms come from the general theory of online convex optimization by \cite{HazanOCO16}. Such algorithms are based on online gradient descent methods.

There is no right answer to the question which category of algorithms is better in practice. Algorithms from both groups theoretically have good performance. Regret (with respect to some offline comparator) is usually bounded by a sublinear function of $T$: 
\begin{enumerate}
\item $R_{T}\leq O(T)$ for Fixed Share with constant share by \cite{HeW98};
\item $R_{T}\leq O(\sqrt{T})$ for Regularized Follow The Leader according to \cite{HazanOCO16};
\item $R_{T}\leq O(\ln T)$ for Fixed Share with decreasing share by \cite{AdaAda16};
\item $R_{T}\leq O(1)$ in aggregating algorithm by \cite{VoV98};
\end{enumerate}
and so on. In fact, the applicability of every particular algorithm and the regret bound depends on the properties of the loss function (convexity, Lipschitz, exponential concavity, mixability, etc.).

When it comes to long-term forecasting, many of these algorithms do not have theoretical guaranties of performance or even do not have a version for the long-term forecasting. Long-term forecasting implies delayed feedback. Thus, the problem of modifying the algorithms for the long-term forecasting can be partially solved by the general results of the theory of forecasting with the delayed feedback.

The main idea in the field of the forecasting with the delayed feedback belongs to \cite{WeO2002}. They studied the simple case of binary sequences prediction under fixed known delay feedback $D$. According to their results, an optimal\footnote{An optimal predictor is any predictor with the regret which is less or equal to the minimax regret. There may exist more than one optimal predictor.} predictor $p^{*}_{D}(x_{t+D}|x_{t},\dots, x_{1})$ for the delay $D$ can be obtained from an optimal predictor $p^{*}_{1}(x_{t+1}|x_{t},\dots,x_{1})$ for the delay $1$. The method implies running $D$ independent copies of predictor $p_{1}^{*}$ on $D$ disjoint time grids
${GR_{d}=\{t\mbox{ }|\mbox{ }t\equiv d\mbox{ }(\mbox{mod } D)\}}$ for $1{\le d\le D}$. Thus, 
$${p^{*}_{D}(x_{t+D}|x_{t},\dots, x_{1})=p_{1}^{*}(x_{t+D}|x_{t},x_{t-D},x_{t-2D},\dots)}$$
for all $t$. We illustrate the optimal partition of the timeline in Figure \ref{figure:grids}.

\begin{figure}[!htb]
\centering
\includegraphics[scale=0.63]{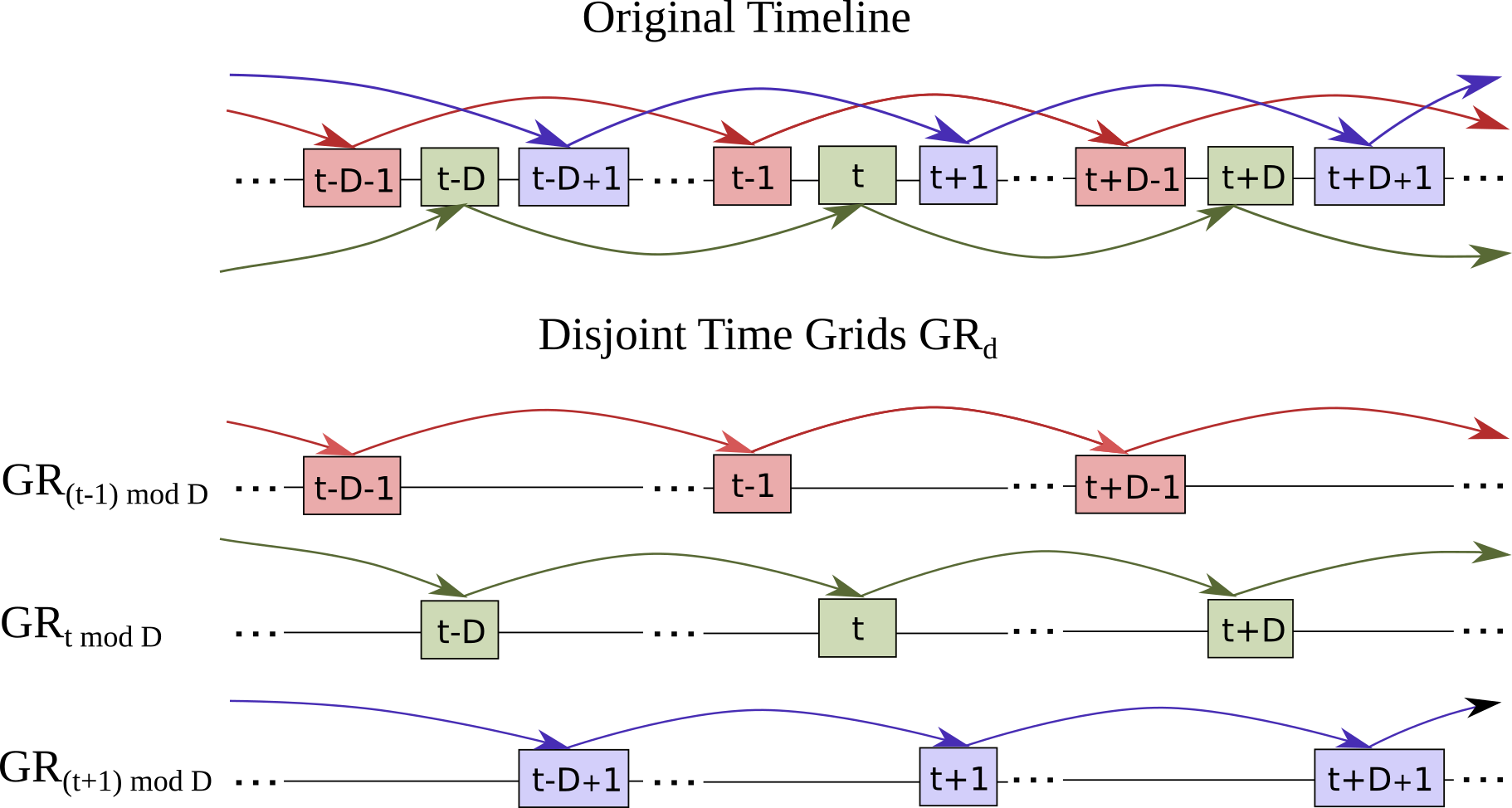}
\caption{The optimal approach to the problem of forecasting with the fixed known delay $D$. The timeline is partitioned into $D$ disjoint grids $GR_{d}$. Games on different grids are considered separately. Each game has fixed known delay $1$ (not $D$).}
\label{figure:grids}
\end{figure}

It turns out, their result also works for the general problem of forecasting under the fixed known delay feedback, in particular, for the prediction with expert advice (we prove this in Appendix \ref{optimal-approach}). Thus, it is easy to apply any $1$-step-ahead forecasting aggregating algorithm to the problem of long-term forecasting by running its $D$ independent copies on $D$ disjoint grids $GR_{d}$ (for $d=1,\dots, D$). We call algorithms obtained by this method as \textit{replicated algorithms}.

Nevertheless, one may say that such a theoretically optimal approach is practically far from optimal because it uses only $\frac{1}{D}$ of observed data at every step of the game. Moreover, separate learning processes on grids $GR_{d}$ (for $d=1,\dots,D$) do not even interact.

Gradient-descent-based aggregating algorithms have several non-replicated adaptations for the long-term forecasting. The most obvious adaptation is the delayed gradient descent by \cite{QuanDGD15}. Also, the problem of prediction with experts' advice can be considered as a special case of online convex optimization with memory by \cite{AHS2015}. Both approaches provide $R_{T}\leq O(\sqrt{T})$ classical regret bound.\footnote{We do not include the value of the forecasting horizon $D$ in regret bounds because in this article we are interested only in the regret asymptotic behavior w.r.t. $T$. In all algorithms that we discuss the asymptotic behavior w.r.t. $D$ is sublinear or linear.} Thus, the practical and theoretical problem of modifying the gradient descent based aggregating algorithms for long-term forecasting can be considered as solved.

In this work, we investigate the problem of modifying aggregating algorithms based on exponential reweighing for the long-term forecasting. We consider the general aggregating algorithm by \cite{Vovk1999} for the $1$-step-ahead forecasting and provide its reasonable non-replicated generalization for the $D$-th-step-ahead forecasting. These algorithms are denoted by $G_{1}$ and $G_{D}$ respectively. We obtain a general expression for the regret bound of $G_{D}$. As an important special case, we consider the classical exponentially reweighing algorithm $V_{1}$ by \cite{VoV98}, that is a case of $G_{1}$, designed to compete with the best expert in the pool. The algorithm $V_{1}$ can be considered close to an optimal one because it provides constant $T$-independent regret bound. We provide its replicated modification $V_{D}$ for the long-term forecasting. We also propose a non-replicated modification $V_{D}^{FC}$ for the long-term forecasting of $V_{1}$ (motivated by reasonable practical approach).\footnote{FC --- full connection.} Our main result here is that the regret bound for $V_{D}^{FC}$ is $O(\sqrt{T})$.

All the algorithms that we develop and investigate require the loss function to be exponentially concave. This is a common assumption (see e.g. \cite{KiW99}) for the algorithms based on the exponential reweighing.\footnote{Usually, even more general assumption is used that the loss function is mixable.}

\vspace{2mm}

\noindent \textbf{The main contributions of this article are the following:}
\begin{enumerate}
\item Developing the general non-replicated exponentially reweighing aggregating algorithm $G_{D}$ for the problem of long-term prediction with experts' advice and estimating its regret.
\item Developing the non-replicated adaptation $V_{D}^{FC}$ of the powerful aggregating algorithm $V_{1}$ by \cite{VoV98}. The obtained algorithm has $O(\sqrt{T})$ regret bound with respect to the best expert in the pool.
\end{enumerate}

In our previous work (see \cite{KVB2018}) we also studied the application of algorithm $V_{1}$ to the long-term forecasting. We applied the method of Mixing Past Posteriors by \cite{BoW2002} to connect the independent learning processes on separate grids $GR_{d}$ (for $d=1,\dots,D$). We obtained the algorithm $V_{D}^{GC}$ that \textit{partially connects} the learning processes on these grids.\footnote{GC --- grid connection.}

In contrast to our previous work, in this article we consider the general probabilistic framework for the long-term forecasting (algorithms $G_{1}$ and $G_{D}$). We obtain the algorithm $V_{D}^{FC}$ that \textit{fully connects} the learning processes on different grids, see details in Subsection \ref{practical-vovk}. 

\vspace{2mm}

\noindent The article is structured as follows.

In Section \ref{prelim} we set up the problem of long-term prediction with experts' advice and state the protocol of the online game.

In Section \ref{1-step-ahead} we discuss the aggregating algorithms for the $1$-step-ahead forecasting. In Subsection \ref{general-model} we describe the general model $G_{1}$ by \cite{Vovk1999} and consider its special case $V_{1}$ in Subsection \ref{vovk-classic}.

In Section \ref{d-step-ahead} we discuss aggregating algorithms for the $D$-th-step-ahead forecasting: we develop general model $G_{D}$ in Subsection \ref{general-model-d}. Then we discuss its two special cases: algorithm $V_{D}$ in Subsection \ref{vovk-optimal} that is a replicated version of $V_{1}$ and our non-replicated version $V_{D}^{FS}$ in Subsection \ref{practical-vovk}. We prove the $O(\sqrt{T})$ regret bound for $V_{D}^{FC}$.

In Appendix \ref{optimal-approach} we generalize the result by \cite{WeO2002} to the case of long-term prediction with experts' advice and prove that the approach with replicating $1$-step-ahead predictors for $D$-th-step-ahead forecasting is optimal.

\section{Preliminaries}
\label{prelim}

We use \textbf{bold} font to denote vectors (e.g. $\bm{w}\in \mathbb{R}^{M}$ for some integer $M$). In most cases, superscript refers to index/coordinate of an element in the vector (e.g. $(w^{1},\dots,w^{N})=\bm{w}$). Subscript is always used to indicate time (e.g. $l_{t}, R_{T}, \omega_{\tau}, w_{t}^{n}$, etc.).

For any integer $M$ we denote the probability simplex of dimension $M$ by 
$${\Delta_{M}=\{\bm{p}\mbox{ such that }(\bm{p}\in \mathbb{R}_{+}^{M})\wedge(\|\bm{p}\|_{1}=1)\wedge(\bm{p}>0)\}}.$$

We use the notation $\bm{e}\in \mathbb{R}_{+}^{M}$ to denote the unit vector $(1,1,\dots,1)$. The dimension $M$ of the vector is always clear from the context. Note that $\frac{\bm{e}}{M}\in\Delta_{M}$.

The words \textbf{prediction} and \textbf{forecasting} are absolute synonyms in this paper.

\subsection{A Game of Long-Term Forecasting with Experts' Advice}
\label{ltf-pea}
We consider the online game of $D$-th-step-ahead forecasting of time series $\omega_{t}\in\Omega$  by aggregating a finite pool of $N$ forecasting experts. We use $\mathcal{N}=\{1,\dots,N\}$ to denote the pool and $n\in\mathcal{N}$ as an index of an expert.

At each integer step $t=1,2,\dots,T-D$ experts $ n\in\mathcal{N}$ present their forecasts $\xi_{t+D}^{ n}\in \Xi$ of time series $\{\omega_{\tau}\}_{\tau=1}^{T}$ for the time moment $t+D$. The master (aggregating) algorithm combines these forecasts into a single (aggregated) forecast ${\gamma_{t+D}\in\Gamma\subset\Xi}$ for the time moment $t+D$. 

After the corresponding outcome $\omega_{t+D}$ is revealed (on the step $t+D$ of the game), both experts and algorithm suffer losses using a loss function ${\lambda:\Omega\times\Xi\rightarrow \mathbb{R}_{+}}$. We denote the loss of expert $ n\in\mathcal{N}$ on step $t+D$ by ${l_{t+D}^{ n}=\lambda(\omega_{t+D},\xi_{t+D}^{ n})}$ and the loss of the aggregating algorithm by ${h_{t+D}=\lambda(\omega_{t+D},\gamma_{t+D})}$. We fix the protocol of the game below.

 \bigskip
 
{\bf Protocol } ($D$-th-step-ahead forecasting with Experts' advice)
{\small
\medskip\hrule\hrule\medskip
\noindent \hspace{2mm}Get the experts $ n\in\mathcal{N}$ predictions $\xi_{t}^{ n}\in\Xi$  for steps $t=1,\dots,D$.

\noindent \hspace{2mm}Compute the aggregated predictions $\gamma_{t}\in\Gamma$  for steps $t=1,\dots,D$.

\smallskip
\noindent \hspace{2mm}{\bf FOR} $t=1,\dots ,T$
\begin{enumerate}
\item Observe the true outcome $\omega_{t}\in \Omega$.
\item Suffer losses from past predictions
\begin{enumerate}

\item Compute the losses ${l_{t}^{ n}=\lambda(\omega_{t},\xi_{t}^{ n})}$ for all $ n\in\mathcal{N}$ of the experts' forecasts $\xi_{t}^{ n}$ made at the step $t-D$.
\item Compute the loss $h_{t}=\lambda(\omega_{t},\gamma_{t})$ of the aggregating algorithm's forecast $\gamma_{t}$ made at the step $t-D$.
\end{enumerate}
\item Make the forecast for the next step (if $t\leq T-1$)
\begin{enumerate}
\item Get the experts $ n\in\mathcal{N}$ predictions $\xi_{t+D}^{ n}\in\Xi$  for the step $t+D$.
\item Compute the aggregated prediction $\gamma_{t+D}\in \Gamma$ of the algorithm.
\end{enumerate}
\end{enumerate}

\noindent \hspace{2mm}{\bf ENDFOR}
\medskip\hrule\hrule\medskip
}
\smallskip

We assume that the forecasts $\xi_{t}^{n}$ of all experts $n\in\mathcal{N}$ for first $D$ time moments $t=1,\dots,D$ are given before the game.

The variables ${L_{T}^{ n}=\sum_{t=1}^{T}l_{t}^{ n}}$ (for all ${ n\in\mathcal{N}}$) and ${H_{T}=\sum_{t=1}^{T}h_{t}}$ correspond to the cumulative losses of expert $ n$ and the aggregating algorithm over the entire game respectively. We also denote the vector of experts' forecasts for the step $t+D$ by $\bm{\xi}_{t+D}=(\xi_{t+D}^{1},\dots,\xi_{t+D}^{N})$.

In the general protocol sets $\Xi$ and $\Gamma\subset\Xi$ may not be equal. For example, the problem of combining $N$ soft classifiers into a hard one has $\Xi=[0,1]$ and $\Gamma=\{0, 1\}\subsetneq \Xi$. In this article we assume that the sets of possible experts' and algorithm's forecasts are equal, that is $\Xi=\Gamma$. Moreover, we assume that $\Xi=\Gamma$ is a convex set. We will not use the notation of $\Xi$ anymore.

The performance of the algorithm is measured by the (cumulative) regret. The cumulative regret is the difference between the cumulative loss of the aggregating algorithm and the cumulative loss of some off-line comparator. A typical approach is to compete with the best expert in the pool. The cumulative regret with respect to the best expert is
\begin{equation}
R_{T}=H_{T}-\min_{ n\in\mathcal{N}}L_{T}^{ n}.
\label{base-regret}
\end{equation}

The goal of the aggregating algorithm is to minimize the regret, that is, ${R_{T}\rightarrow \min}$. In order to theoretically guarantee algorithm's performance, some upper bound is usually proved for the cumulative regret ${R_{T}\leq f(T)}$. 

In the base setting \eqref{base-regret}, sub-linear upper bound $f(T)$ for the regret leads to asymptotic performance of the algorithm equal to the performance of the best expert. More precisely, we have $\lim_{T\rightarrow \infty}\frac{R_{T}}{T}=0$.

\subsection{Exponentially concave loss functions}

We investigate learning with exponentially concave loss functions. Loss function ${\lambda:\Omega\times\Gamma\rightarrow \mathbb{R}_{+}}$ is called $\eta$-exponentially concave (for some $\eta>0$) if for all $\omega\in\Omega$ and all probability distributions $\bm{\pi}$ on set $\Gamma$ the following holds true:

\begin{equation}
e^{-\eta\lambda(\omega,\gamma_{\bm{\pi}})}\geq \int_{\gamma\in\Gamma}e^{-\eta\lambda(\omega, \gamma)}\bm{\pi}(d\gamma),
\label{exp-concavity}
\end{equation}
where 
\begin{equation}\gamma_{\bm{\pi}}=\int_{\gamma\in\Gamma}\gamma  \bm{\pi}(d\gamma)=\mathbb{E}_{\bm{\pi}}\gamma.
\label{wa-prediction}
\end{equation}
In \eqref{exp-concavity} variable $\gamma_{\bm{\pi}}$ is called aggregated prediction. Since $\Gamma$ is convex, we have $\gamma_{\bm{\pi}}\in\Gamma$.

If a loss function is $\eta$-exponentially concave, it is also $\eta'$-exponentially concave for all $\eta'\in (0, \eta]$. This fact immediately follows from the general properties of exponentially concave functions. For more details see the book by \cite{HazanOCO16}.

Note that the basic square loss and the log loss functions are both exponentially concave. This fact is proved by \cite{KiW99}.

\section{Aggregating Algorithm for 1-Step-Ahead Forecasting}
\label{1-step-ahead}

In this section we discuss basic aggregating algorithms for $1$-step-ahead forecasting based on exponential reweighing. Our framework is built on the general aggregating algorithm $G_{1}$ by \cite{Vovk1999}, we discuss it in Subsection \ref{general-model}. The simplest and earliest version $V_{1}$ by \cite{VoV98} of this algorithm is discussed in Subsection \ref{vovk-classic}.

\subsection{General Model}
\label{general-model}

We investigate the adversarial case, that is, no assumptions (functional, stochastic, etc) are made about the nature of data and experts. However, it turns out that in this case it is convenient to develop algorithms using some probabilistic interpretations. 

Recall that $\bm{\xi}_{t}=(\xi_{t}^{1},\dots,\xi_{t}^{N})$ is a vector of experts' predictions of $\omega_{t}$. Loss function $\lambda:\Omega\times\Gamma\rightarrow \mathbb{R}_{+}$ is $\eta$-exponentially concave for some $\eta>0$.

We assume that data is generated using some probabilistic model with hidden states. The model is shown in Figure \ref{figure:model-general}.

\begin{figure}[!htb]
\begin{center}
\includegraphics[scale=0.7]{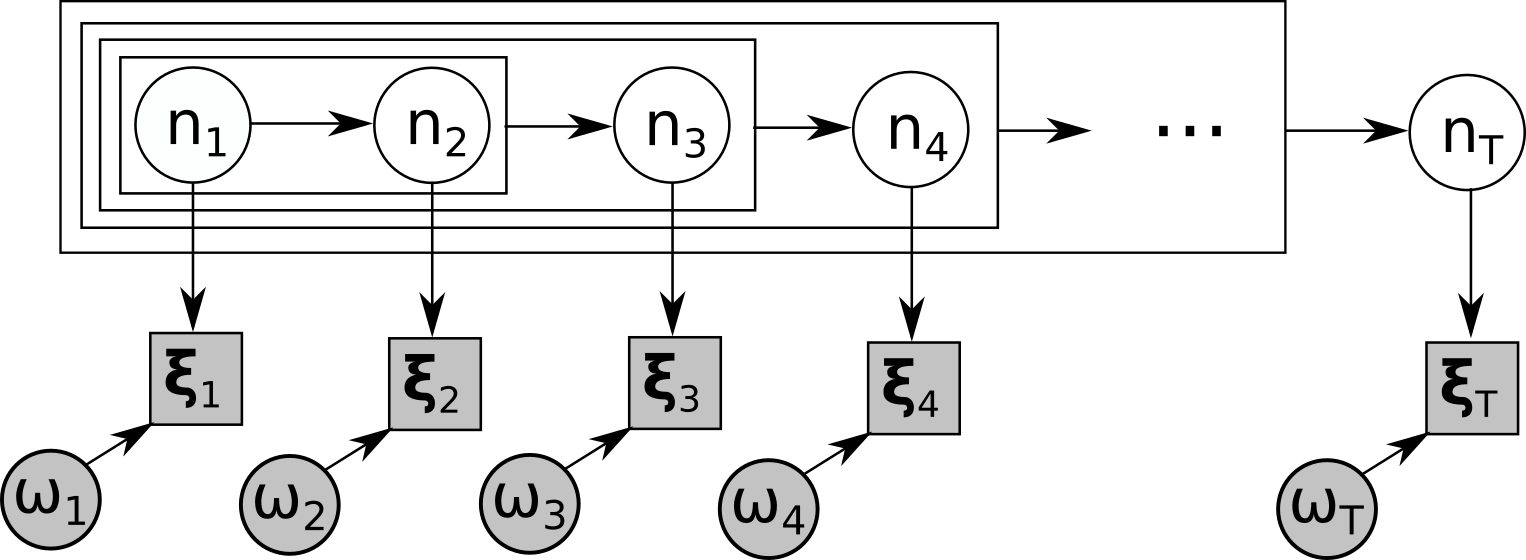}
\end{center}
\caption{Probabilistic model of data generation process.}
\label{figure:model-general}
\end{figure} 

We suppose that there is some hidden sequence of experts $ n_{t}\in \mathcal{N}$ (for $t=1,2,\dots,T$) that generates the experts predictions $\bm{\xi}_{t}$. Particular, hidden expert $n_{t}$ at step $t$ is called active expert. The conditional probability to observe the vector $\bm{\xi}_{t}$ of experts' predictions at step $t$ is

$$p(\bm{\xi}_{t}| n_{t})=\frac{e^{-\eta \lambda(\omega_{t},\xi_{t}^{ n_{t}})}}{Z_{t}},$$
where $Z_{t}=\int_{\xi\in\Gamma}e^{-\eta\lambda(\omega_{t},\xi)}d\xi$ is the normalizing constant.\footnote{Constant $Z_{t}$ is $ n_{t}$-independent. In the article we do not need to compute the exact value of the normalizing constant $Z_{t}$.} We denote $\bm{\Xi}_{t}=(\bm{\xi}_{1},\dots,\bm{\xi}_{t})$ and $\Omega_{t}=(\omega_{1},\dots,\omega_{t})$ for all $t=1,\dots, T$.

For the first active expert $ n_{1}$ some known prior distribution is given $p( n_{1})=p_{0}( n_{1})$. The sequence $( n_{1},\dots, n_{T})$ of active experts is generated step-by-step. For $t\in\{1,\dots,T-1\}$ each $ n_{t+1}$ is sampled from some known distribution $p( n_{t+1}|\mathcal{N}_{t})$, where $\mathcal{N}_{t}=( n_{1},\dots, n_{t})$.\footnote{In case $p( n_{t+1}|\mathcal{N}_{t})=p( n_{t+1}| n_{t})$, we obtain traditional Hidden Markov Process. The hidden state at step $t+1$ depends only on the previous hidden state at step $t$.} Thus, active expert $ n_{t+1}$ depends on the previous experts $\mathcal{N}_{t}$.

The considered probabilistic model is:\footnote{The correct way is to include the time series values $\omega_{t}$ as the conditional parameter in the model probabilistic distribution, that is, $p(\mathcal{N}_{T},\bm{\Xi}_{T}|\Omega_{T})$. We omit the values $\omega_{t}$ in probabilities $p(\cdot)$ in order not to overburden the notation.}

\begin{equation}
p(\mathcal{N}_{T},\bm{\Xi}_{T})=p(\mathcal{N}_{T})\cdot p(\bm{\Xi}_{T}|\mathcal{N}_{T})=\bigg[p_{0}( n_{1})\prod_{t=2}^{T}p( n_{t}|\mathcal{N}_{t-1})\bigg]\cdot \bigg[\prod_{t=1}^{T}p(\bm{\xi}_{t}| n_{t})\bigg].
\label{model-general}
\end{equation}

The similar equation holds true for $t\leq T$:

$$
p(\mathcal{N}_{t},\bm{\Xi}_{t})=p(\mathcal{N}_{t})\cdot p(\bm{\Xi}_{t}|\mathcal{N}_{t})=\bigg[p_{0}( n_{1})\prod_{\tau=2}^{t}p(n_{\tau}|\mathcal{N}_{\tau-1})\bigg]\cdot \bigg[\prod_{\tau=1}^{t}p(\bm{\xi}_{\tau}| n_{\tau})\bigg].
$$

The probability $p(\mathcal{N}_{T})$ is that of hidden states (active experts).\footnote{The form 
$p(\mathcal{N}_{T})=p_{0}( n_{1})\prod_{t=2}^{t}p( n_{t}|\mathcal{N}_{t-1})$ is used only for convenience and association with online scenario. It does not impose any restrictions on the type of probability distribution. In fact, $p(\mathcal{N}_{T})$ may be any distribution on $\mathcal{N}^{T}$ of any form.}

Suppose that the current time moment is $t$. We observe the experts' predictions $\bm{\Xi}_{t}$ made earlier, time series $\Omega_{t}$ and predictions $\bm{\xi}_{t+1}$ for the step $t+1$. Since we observe $\Omega_{t}$ and $\bm{\Xi}_{t}$, we are able to estimate the conditional distribution $p(\mathcal{N}_{t}|\bm{\Xi}_{t})$ of hidden variables $\mathcal{N}_{t}$. This estimate allows us to compute the conditional distribution on the active expert $n_{t+1}$ at the moment $t+1$. We denote for all $ n_{t+1}\in\mathcal{N}$

\begin{equation}
w_{t+1}^{ n_{t+1}}=p( n_{t+1}|\bm{\Xi}_{t})=\sum_{\mathcal{N}_{t}\in\mathcal{N}^{t}}p(n_{t+1}|\mathcal{N}_{t})p(\mathcal{N}_{t}|\bm{\Xi}_{t}).
\label{general-update}
\end{equation}

We use the weight vector $\bm{w}_{t+1}=(w_{t+1}^{1},\dots,w_{t+1}^{N})$ to combine the aggregated prediction for the step $t+1$:

$$\gamma_{t+1}=\sum_{n_{t+1}\in\mathcal{N}}\xi_{t+1}^{n_{t+1}}p( n_{t+1}|\bm{\Xi}_{t})=\langle \bm{w}_{t+1}, \bm{\xi}_{t+1}\rangle=\sum_{n=1}^{N}w_{t+1}^{n} \xi_{t+1}^{n}.$$

The aggregating algorithm is shown below. We denote it by $G_{1}=G_{1}(p)$, where $p$ indicates the probability distribution $p(\mathcal{N}_{T})$ of active experts to which the algorithm is applied.

 \bigskip
 
{\bf Algorithm $G_{1}(p)$ } (Aggregating algorithm for distribution $p$ of active experts)
{\small
\medskip\hrule\hrule\medskip
\noindent \hspace{2mm}Set initial prediction weights $\bm{w}_{1}^{ n_{1}}=p_{0}( n_{1})$.

\noindent \hspace{2mm}Get the experts $ n\in\mathcal{N}$ predictions $\xi_{1}^{n}\in\Gamma$  for the step $t=1$.

\noindent \hspace{2mm}Compute the aggregated prediction $\gamma_{1}=\langle\bm{w}_{1},\bm{\xi}_{1}\rangle$ for the step $t=1$.

\smallskip
\noindent \hspace{2mm}{\bf FOR} $t=1,\dots ,T$
\begin{enumerate}
\item Observe the true outcome $\omega_{t}\in \Omega$.
\item Update the weights
\begin{enumerate}

\item Calculate the prediction weights $\bm{w}_{t+1}=(w_{t}^{1},\dots, w_{t}^{N})$, where 
$$w_{t}^{n_{t+1}}=p(n_{t+1}|\bm{\Xi}_{t})$$
for all $n_{t+1}\in\mathcal{N}$.
\end{enumerate}
\item Make forecast for the next step (if $t\leq T-1$)
\begin{enumerate}
\item Get the experts $ n\in\mathcal{N}$ predictions $\xi_{t+1}^{ n}\in\Gamma$  for the step $t+1$.
\item Combine the aggregated prediction $\gamma_{t+1}=\langle \bm{w}_{t+1}, \bm{\xi}_{t+1}\rangle\in \Gamma$ of the algorithm.
\end{enumerate}
\end{enumerate}

\noindent \hspace{2mm}{\bf ENDFOR}
\medskip\hrule\hrule\medskip
}
\smallskip

To estimate the performance of the obtained algorithm, we prove Theorem \ref{theorem-main}. Recall that $H_{T}$ is the cumulative loss of the algorithm over the entire game.

\begin{theorem}
\label{theorem-main}
For the algorithm $G_{1}$ applied to to the model \eqref{model-general} the following upper bound for the cumulative loss over the entire game holds true:

\begin{eqnarray}
H_{T}\leq -\frac{1}{\eta}\ln \bigg[\mathbb{E}_{p(\mathcal{N}_{T})}\big[e^{-\eta L_{T}^{\mathcal{N}_{T}}}\big]\bigg].
\label{main-loss-bound}\end{eqnarray}
\end{theorem}

Similar results were obtained by \cite{Vovk1999}. In this article we reformulate these results in terms of our interpretable probabilistic framework. This is required for the completeness of the exposition and theoretical analysis of algorithm $G_{D}$ (see Subsection \ref{general-model-d}).

\begin{proof}
Define the mixloss at the step $t$:
\begin{equation}
m_{t}=-\frac{1}{\eta}\ln\big[\sum_{ n_{t}\in\mathcal{N}}e^{-\eta\lambda(\omega_{t},\xi_{t}^{ n_{t}})}\cdot w_{t}^{ n_{t}}\big].
\label{mixloss-definition}
\end{equation}
Since $\lambda$ is $\eta$-exponentially concave function, for the aggregated prediction $\gamma_{t}=\langle\bm{w}_{t},\bm{\xi}_{t}\rangle$ and probability distribution $\bm{w}_{t}$ we have
$$e^{-\eta\lambda(\omega_{t},\gamma_{t})}\geq \sum_{ n_{t}\in\mathcal{N}}e^{-\eta\lambda(\omega_{t},\xi_{t}^{ n_{t}})}w_{t}^{ n_{t}},$$
which is equal to $e^{-\eta h_{t}}\geq e^{-\eta m_{t}}$. We conclude that $h_{t}\leq m_{t}$ for all $t\in \{1,\dots,T\}$. Thus, the similar inequality is true for the cumulative loss of the algorithm and the cumulative mixloss:

$$H_{T}=\sum_{t=1}^{T}h_{t}\leq \sum_{t=1}^{T}m_{t}=M_{T}.$$
Now lets compute $M_{T}$.

For all $t$
\begin{eqnarray}
m_{t}=-\frac{1}{\eta}\ln\big[\sum_{ n_{t}\in\mathcal{N}}e^{-\eta\lambda(\omega_{t},\xi_{t}^{ n_{t}})}\cdot p( n_{t}|\bm{\Xi}_{t-1})\big]=
\nonumber
\\
-\frac{1}{\eta}\ln\big[\sum_{ n_{t}\in\mathcal{N}}Z_{t}\cdot p(\bm{\xi}_{t}| n_{t})\cdot p( n_{t}|\bm{\Xi}_{t-1})\big]=
\nonumber
\\
-\frac{1}{\eta}\ln Z_{t}-\frac{1}{\eta}\ln\big[\sum_{ n_{t}\in\mathcal{N}}p(\bm{\xi}_{t}| n_{t})\cdot p( n_{t}|\bm{\Xi}_{t-1})\big]=
\nonumber
\\
-\frac{1}{\eta}\ln Z_{t}-\frac{1}{\eta}\ln p(\bm{\xi}_{t}|\bm{\Xi}_{t-1})
\nonumber
\end{eqnarray}
We compute
\begin{eqnarray}
H_{T}\leq M_{T}=\sum_{t=1}^{T}m_{t}=\frac{1}{\eta}\ln \prod_{t=1}^{T}Z_{t}-\frac{1}{\eta}\ln\prod_{t=1}^{T}p(\bm{\xi}_{t}|\bm{\Xi}_{t-1})=
\nonumber
\\
\frac{1}{\eta}\ln \prod_{t=1}^{T}Z_{t}-\frac{1}{\eta}\ln p(\bm{\Xi}_{T})=-\frac{1}{\eta}\ln \bigg[\mathbb{E}_{p(\mathcal{N}_{T})}\big[e^{-\eta L_{T}^{\mathcal{N}_{T}}}\big]\bigg]
\end{eqnarray}
and finish the proof.
\end{proof}

In the current form it is difficult to understand the meaning of the theorem. However, the main idea is partially shown in the following corollary.

\begin{corollary}\label{corollary-finite}
The regret of the algorithm $G_{1}(p)$ with respect to the sequence $\mathcal{N}_{T}^{*}=\{n_{1}^{*},\dots,n_{T}^{*}\}$ of experts has the following upper bound:
\begin{eqnarray}
R_{T}(\mathcal{N}_{T}^{*})=H_{T}-\sum_{t=1}^{T}l_{t}^{ n_{t}^{*}}\leq -\frac{1}{\eta}\ln p(\mathcal{N}_{T}^{*}).
\label{bound-to-single-expert}
\end{eqnarray}
\end{corollary}
\begin{proof}
We simply bound \eqref{main-loss-bound}:
$$H_{T}\leq -\frac{1}{\eta}\ln \bigg[\mathbb{E}_{p(\mathcal{N}_{T})}\big[e^{-\eta L_{T}^{\mathcal{N}_{T}}}\big]\bigg]\leq -\frac{1}{\eta}\ln \bigg[p(\mathcal{N}_{T}^{*})\cdot e^{-\eta L_{T}^{\mathcal{N}_{T}^{*}}}\bigg]=\sum_{t=1}^{T}l_{t}^{ n_{t}^{*}}-\frac{1}{\eta}\ln p(\mathcal{N}_{T}^{*}).$$
\end{proof}

Applying algorithm $G_{1}$ to different probability models $p$ makes it possible to change the upper regret bound with respect to concrete sequences $\mathcal{N}_{T}$.\footnote{In this article we do not raise the question of computational efficiency of algorithm $G_{1}(p)$. In fact, computational time and required memory can be pretty high for complicated distributions $p$, even $O(N^{T})$. Nevertheless, all the special algorithms that we consider ($V_{1}, V_{D}, V_{D}^{FC}$) are computationally efficient. They require $\leq O(NT)$ computational time and $\leq O(ND)$ memory.} Choosing different $p$ makes it possible to obtain adaptive algorithms (e.g. Fixed Share by \cite{HeW98}, \cite{AdaAda16}, \cite{Vovk1999}). Such algorithms provide low regret not only with respect to the best constant sequence of experts, but also with some more complicated sequences.

For example, suppose we are to obtain minimal possible regret bound with respect to the sequences $(\mathcal{N}^{T})^{*}\subset \mathcal{N}^{T}$. In this case it is reasonable to set $p(\mathcal{N}_{T})=\frac{1}{|(\mathcal{N}^{T})^{*}|}$ for all $\mathcal{N}_{T}\in(\mathcal{N}^{T})^{*}$ and $p(\mathcal{N}_{T})=0$ for all other $\mathcal{N}_{T}$. 

Nevertheless, the most popular approach is to compete with the best (fixed) expert. This approach is simple and, at the same time, serves as the basis for each research. 

\subsection{Case of Hidden Markov Process: Classical Vovk's Algorithm}
\label{vovk-classic}
Consider the simplified dependence of active experts shown on Figure \ref{figure:model-simple}. For all $t$ expert $ n_{t+1}$ depends only on the previous expert $ n_{t}$, that is, $p( n_{t+1}|\mathcal{N}_{T})=p( n_{t+1}| n_{t})$.

\begin{figure}[!htb]
\begin{center}
\includegraphics[scale=0.7]{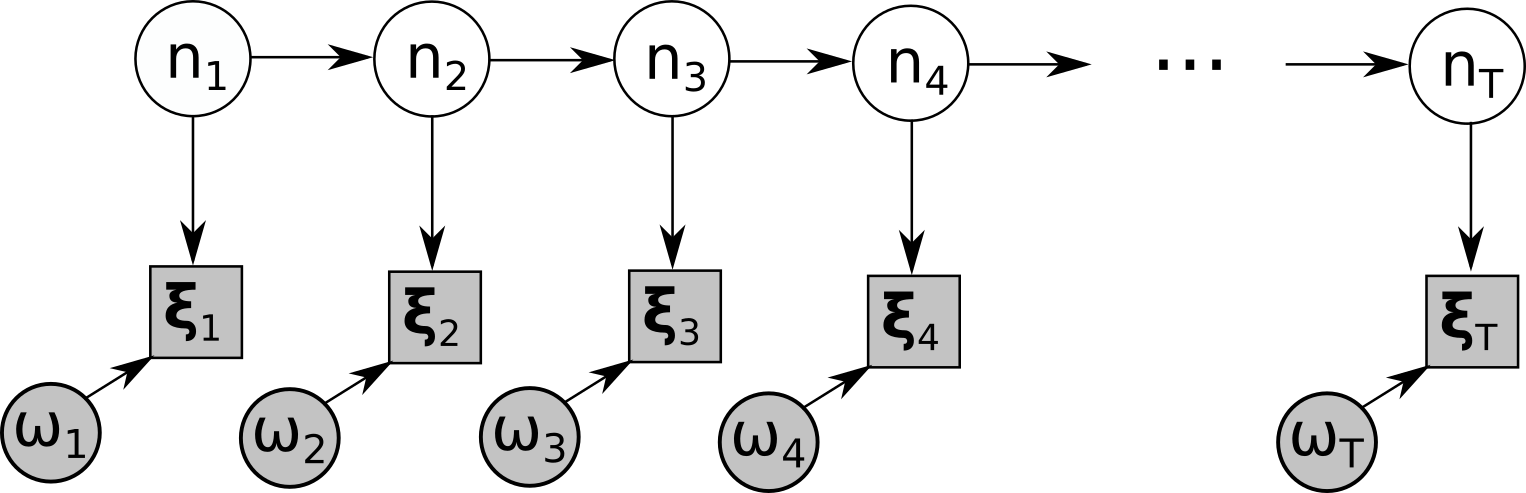}
\end{center}
\caption{Hidden Markov Model model for data generation process.}
\label{figure:model-simple}
\end{figure} 

The classic Vovk's aggregating algorithm is obtained from $G_{1}$ by applying it to the simple distribution $p$. We put $p(n_{1})=\frac{1}{N}$ (for all $n_{1}\in\mathcal{N}$) and $p(n_{t}|n_{t-1})=[n_{t}=n_{t-1}]$ (for all $t>1$ and $n_{t},n_{t-1}\in \mathcal{N}$). We denote the algorithm $G_{1}$ for the described $p$ by $V_{1}$.

According to Corollary \ref{corollary-finite}, algorithm $V_{1}$ has the following regret bound with respect to the best constant expert $\mathcal{N}_{T}^{*}=(n^{*},\dots, n^{*})$:

\begin{equation}
R_T(\mathcal{N}_{T}^{*})=H_{T}-\sum_{t=1}^{T}l_{t}^{ n^{*}}\leq -\frac{1}{\eta}\ln p(\mathcal{N}_{T}^{*})=\frac{\ln N}{\eta}.
\label{vovk-regret-bound}
\end{equation}

At the same time, it is easy to recurrently compute the weights $\bm{w}_{t}$ step by step, that is, $\bm{w}_{1}\rightarrow \bm{w}_{2}\rightarrow \dots$. We get

$$w_{t+1}^{n_{t+1}}=p(n_{t}|\bm{\Xi}_{t})=\frac{e^{-\eta L_{t}^{n_{t}}}}{\sum_{n=1}^{N}e^{-\eta L_{T}^{n}}}=\frac{w_{t}^{n_{t}}e^{-\eta l_{t}^{n_{t}}}}{\sum_{n=1}^{N}w_{t}^{n}e^{-\eta l_{t}^{n}}}.$$

Thus, algorithm $V_{1}$ is a powerful and computationally efficient tool to aggregate experts in the game of the $1$-step-ahead forecasting.

\section{Aggregating Algorithm for Long-Term Forecasting}
\label{d-step-ahead}
The Section is devoted to aggregating algorithms for the long-term forecasting. In Subsection \ref{general-model-d} we provide a natural long-term forecasting generalization $G_{D}$ of algorithm $G_{1}$ by \cite{Vovk1999}. We provide the general regret bound and discuss the difficulties that prevent us from obtaining general bound in a simple form. In Subsection \ref{vovk-optimal} we show how the replicated version $V_{D}$ for the long-term forecasting of $V_{1}$ fits into the general model, and prove its regret bound. In Subsection \ref{practical-vovk} we describe non-replicated version $V_{D}^{FC}$ for the long-term forecasting of $V_{1}$ and prove its $O(\sqrt{T})$ regret bound.

\subsection{General Model}
\label{general-model-d}
We describe the natural algorithm obtained by enhancing $G_{1}$ for the problem of the $D$-th-step-ahead forecasting below and denote it by $G_{D}$. Note that weights $\bm{w}_{t}$ in $G_{D}$ differ for different $D$ (for the same probability model $p$).

 \bigskip
 
{\bf Algorithm $G_{D}(p)$ } (Aggregating algorithm for distribution $p$ of active experts)
{\small
\medskip\hrule\hrule\medskip
\noindent \hspace{2mm}Set initial prediction weights $w_{t}^{n_{t}}=p(n_{t})$ for all $t=1,\dots,D$ and $n_{t}\in\mathcal{N}$.

\noindent \hspace{2mm}Get the predictions $\xi_{t}^{n}\in\Gamma$ of experts $ n\in\mathcal{N}$  for steps $t=1,\dots,D$.

\noindent \hspace{2mm}Compute the aggregated predictions $\gamma_{t}=\langle\bm{w}_{t},\bm{\xi}_{t}\rangle$ for steps $t=1,\dots,D$.

\smallskip
\noindent \hspace{2mm}{\bf FOR} $t=1,\dots ,T$
\begin{enumerate}
\item Observe the true outcome $\omega_{t}\in \Omega$.
\item Update the weights
\begin{enumerate}

\item Calculate the prediction weights $\bm{w}_{t+D}=(w_{t}^{1},\dots, w_{t}^{N})$, where 
$$w_{t}^{n_{t+D}}=p(n_{t+D}|\bm{\Xi}_{t})$$
for all $n_{t+D}\in\mathcal{N}$.
\end{enumerate}
\item Make forecast for the next step (if $t\leq T-1$)
\begin{enumerate}
\item Get the predictions $\xi_{t+D}^{ n}\in\Gamma$ of experts $ n\in\mathcal{N}$ for the step $t+D$.
\item Combine the aggregated prediction $\gamma_{t+D}=\langle \bm{w}_{t+D}, \bm{\xi}_{t+D}\rangle\in \Gamma$ of the algorithm.
\end{enumerate}
\end{enumerate}

\noindent \hspace{2mm}{\bf ENDFOR}
\medskip\hrule\hrule\medskip
}
\smallskip

Despite the fact that algorithm $G_{D}$ (for $D>1$) is the direct modification of $G_{1}$, it seems hard to obtain the adequate general bound of the loss of the form \eqref{main-loss-bound}. Indeed, let us try to apply the ideas of Theorem \ref{theorem-main} proof to algorithm $G_{D}$.

Denote by $m_{t}$ the mixloss from \eqref{mixloss-definition}. Recall that the weights $w_{t}^{n_{t}}$ here are equal to the probabilities $p(n_{t}|\bm{\Xi}_{t-D})$ but not $p(n_{t}|\bm{\Xi}_{t-1})$. Again, for $\eta$-exponentially concave function we have $h_{t}\leq m_{t}$. Similar to the proof of Theorem \ref{theorem-main}, we compute for all $t$
$$m_{t}=-\frac{1}{\eta}\ln Z_{t}-\frac{1}{\eta}\ln p(\bm{\xi}_{t}|\bm{\Xi}_{t-D}),$$
where we assume $\bm{\Xi}_{t-D}=\varnothing$ for $t\leq D$. The cumulative mixloss is equal to

\begin{equation}H_{t}\leq M_{t}=\sum_{t=1}^{T}m_{t}=\frac{1}{\eta}\ln \prod_{t=1}^{T}Z_{t}-\frac{1}{\eta}\ln\prod_{t=1}^{T}p(\bm{\xi}_{t}|\bm{\Xi}_{t-D}).
\label{main-loss-bound-2}
\end{equation}
Unfortunately, for $D>1$ in the general case this expression can not be simplified in the same way as in Theorem \ref{theorem-main} for $G_{1}$.

However, there exist some simple theoretical cases when this bound can be simplified. For a fixed $D$, the obvious one is when

\begin{equation}
p(\mathcal{N}_{T})=\prod_{t=T-D+1}^{T}p(\widehat{\mathcal{N}}_{t}),
\label{full-disconnection}
\end{equation}
where we use $\widehat{\mathcal{N}}_{t}=(\dots, n_{t-2D},n_{t-D},n_{t})$ for all $t=1,\dots, T$. Note that \eqref{full-disconnection} means that probability distributions on separate grids $GR_{d}$ (for $d=1,\dots, D$) are independent. In this case, the learning process separates into $D$ disjoint one-step ahead forecasting games on grids $GR_{d}$. We have $p(\bm{\xi}_{t}|\bm{\Xi}_{t-D})=p(\bm{\xi}_{t}|\widehat{\bm{\Xi}}_{t-D})$ and \eqref{main-loss-bound-2} is simplified to

\begin{eqnarray}
H_{t}\leq M_{t}=\frac{1}{\eta}\ln \prod_{t=1}^{T}Z_{t}-\frac{1}{\eta}\ln\prod_{t=T-D+1}^{T}p(\widehat{\bm{\Xi}}_{t})=
\nonumber
\\
\frac{1}{\eta}\ln \prod_{t=1}^{T}Z_{t}-\frac{1}{\eta}\ln\prod_{t=T-D+1}^{T}p(\widehat{\bm{\Xi}}_{t})=-\frac{1}{\eta}\ln \bigg[\mathbb{E}_{p(\mathcal{N}_{T})}\big[e^{-\eta L_{T}^{\mathcal{N}_{T}}}\big]\bigg],
\nonumber
\end{eqnarray}
where $\widehat{\bm{\Xi}}_{t}=\{\bm{\xi}_{t},\bm{\xi}_{t-D}, \bm{\xi}_{t-2D},\dots\}$ for all $t$.

In Appendix \ref{optimal-approach} we prove that the approach \eqref{full-disconnection} may be considered as optimal when the goal is to compete with the best expert. Nevertheless, there are no guaranties that this approach is optimal in the general case. 

\subsection{Optimal Approach: Replicated Vovk's Algorithm}
\label{vovk-optimal}

Algorithm $V_{1}$ can be considered as close to optimal (for the $1$-step-ahead forecasting and competing with the best expert in the pool) because of its low constant regret bound $R_{T}\leq \frac{\ln N}{\eta}$. One can obtain the aggregating algorithm $V_{D}$ for the long-term forecasting whose regret is close to optimal when competing with the best constant active expert. The idea is to run $D$ independent one-step-ahead forecasting algorithms $V_{1}$ on $D$ disjoint subgrids $GR_{d}$ (for $d\in\{1,\dots,D\}$). This idea is motivated by Theorem \ref{optimality-theorem} from Appendix \ref{optimal-approach}.

To show how this method fits into the general model $G_{D}$, we define $p(\mathcal{N}_{D})\equiv (\frac{1}{N})^{D}$ for all $\mathcal{N}_{D}\in \mathcal{N}^{D}$. Next, for all $t>D$ we set $p(n_{t+1}|\mathcal{N}_{t})=p(n_{t+1}|n_{t+1-D})=[n_{t+1}=n_{t+1-D}]$, that is, active expert $n_{t+1}$ depends only on expert $n_{t+1-D}$ that was active $D$ steps ago. The described model is shown in Figure \ref{figure:model-D}.

\begin{figure}[!htb]
\centering
\includegraphics[scale=0.63]{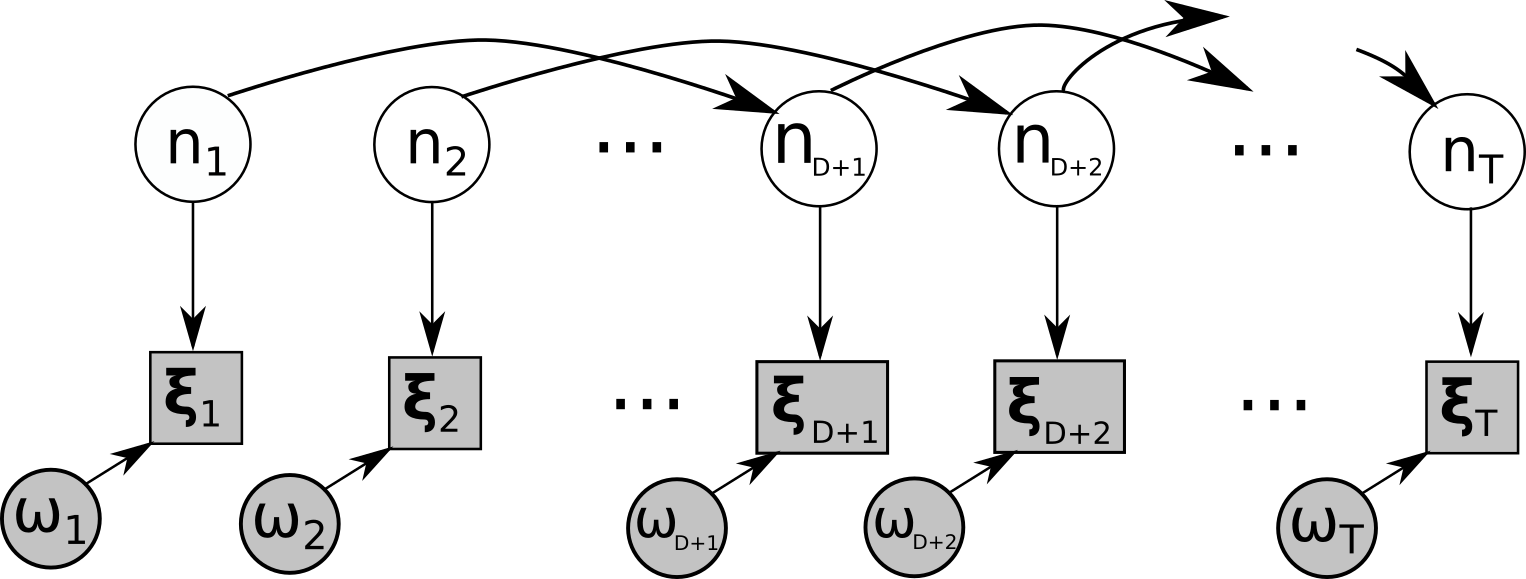}
\caption{The probabilistic model for algorithm $V_{D}$.}
\label{figure:model-D}
\end{figure}

For simplicity, we assume that $T$ is a multiple of $D$. Since $V_{D}$ runs $D$ independent copies of $V_{1}$, it is easy to estimate its regret with respect to the best expert:

$$R_T(\mathcal{N}_{T}^{*})=H_{T}-\sum_{t=1}^{T}l_{t}^{ n^{*}}\leq -\frac{1}{\eta}\ln p(\mathcal{N}_{T}^{*})=D\frac{\ln N}{\eta}.$$

At the same time, the weights' updating process is simple (it can be separated into grids). We have

$$w_{t}^{n_{t}}=p(n_{t}|\bm{\Xi}_{t-D})=\frac{e^{-\eta \widehat{L}_{t-D}^{n_{t}}}}{\sum_{n=1}^{N}e^{-\eta \widehat{L}_{t-D}^{n}}}=\frac{w_{t-D}^{n_{t}}e^{-\eta l_{t-D}^{n_{t}}}}{\sum_{n=1}^{N}w_{t-D}^{n}e^{-\eta l_{t-D}^{n}}},$$
where $\widehat{L}_{t}^{n_{t}}=\sum_{\tau=1}^{t}l_{t}^{n_{t}}$ for all $t$. This formula allows efficient recurrent computation 
$$\dots\rightarrow \bm{w}_{t-D}\rightarrow \bm{w}_{t}\rightarrow \bm{w}_{t+D}\rightarrow \dots.$$

\subsection{Practical Approach: Non-Replicated Vovk's Algorithm}
\label{practical-vovk}

Despite the fact that algorithm $V_{D}$ is theoretically close to optimal for competing with the best expert, it has several practically obvious disadvantages. With the increase of $D$ the overall length of subgames decreases ($\sim \frac{T}{D}$) and subgrids become more infrequent.
Moreover, to set the forecasting weight $\bm{w}_{t+D}$ at step $t$ we use only $\approx\frac{t}{D}$ previous observations and forecasts.

One may wonder why not use all the observed losses to set the weight $\bm{w}_{t+D}$. We apply the algorithm $G_{D}$ to the probability distribution $p$ from Subsection \ref{vovk-classic} which is a case of model from Figure \ref{figure:model-simple}. We denote the obtained algorithm by $V_{D}^{FC}$.

The weights are efficiently recomputed $\dots\rightarrow \bm{w}_{t-1}\rightarrow \bm{w}_{t}\rightarrow \bm{w}_{t+1}\rightarrow \dots$ according to the formula

$$w_{t}^{n_{t}}=p(n_{t}|\bm{\Xi}_{t-D})=\frac{e^{-\eta L_{t-D}^{n_{t}}}}{\sum_{n=1}^{N}e^{-\eta L_{t-D}^{n}}}=\frac{w_{t-1}^{n_{t}}e^{-\eta l_{t-D}^{n_{t}}}}{\sum_{n=1}^{N}w_{t-1}^{n}e^{-\eta l_{t-D}^{n}}}.$$

\begin{theorem} For the $\eta^{\lambda}$-exponentially concave (for some $\eta^{\lambda}$) and $L$-Lipshitz for all $\omega$ w.r.t. second argument $\gamma\in\Gamma$ and $\|\cdot\|_{\Gamma}$ loss function 
$$\lambda:\Omega\times\Gamma\rightarrow [0, H]\subset\mathbb{R}_{+}$$
with $\max_{\gamma\in\Gamma}\|\gamma\|_{\Gamma}\leq B$ there exists $T_{0}$ such that for all $T\geq T_{0}$ the following holds true: there exists $\eta^{\lambda}>\eta^{*}>0$ such that algorithm $V_{D}^{FC}=G_{D}(p)$ with $N$ experts and learning rate $\eta^{*}$ has regret bound 
$$R_{T}\leq O(\sqrt{N\ln N}\cdot \sqrt{T})$$
with respect to the best expert in the pool.
\label{theorem-practical-vovk}
\end{theorem}

\begin{proof}
We use the superscript $(\ldots)^{D}$ to denote the variables obtained by algorithm $G_{D}(p)$ (for example, weights $\bm{w}_{t}^{D}$, predictions $\gamma_{t}^{D}$, etc.). Our main idea is to prove that the weights $\bm{w}_{t}^{D}$ are approximately equal to the weights $\bm{w}_{t}^{1}$ obtained in the one-step-ahead forecasting game $G_{1}(p)$ with the same experts and same time series $\omega_{t}$. Thus, the forecasts $\gamma_t^{D}$ and $\gamma_{t}^{1}$ have approximately the same losses $h_{t}^{D}$ and $h_{t}^{1}$ respectively.

We compare both algorithms with the same (yet unknown) learning rate $\eta$. Note that $\bm{w}_{t+1}^{1}=\bm{w}_{t+D}^{D}$. For all $t$ we have 
$$(w_{t}^{n_{t}})^{1}=(w_{t+D-1}^{n_{t}})^{D}\sim (w_{t}^{n_{t}})^{D}\cdot e^{-\eta L_{[t-D+1, t-1]}^{n_{t}}},$$
where $L_{[t-D+1,t-1]}^{n_{t}}=\sum_{\tau=t-D+1}^{t-1}l_{\tau}^{n_{t}}$. We estimate the difference between the cumulative losses $H_{T}^{1}$ and $H_{T}^{D}$ of forecasts of aggregating algorithms $G_{1}(p)=V_{1}$ and $G_{D}(p)=V_{D}^{FC}$ for the given $p$ from Subsection \ref{vovk-classic}.

\begin{eqnarray}
|H_{T}^{1}-H_{T}^{D}|=|\sum_{t=1}^{T}h_{t}^{1}-\sum_{t=1}^{T}h_{t}^{D}|\leq \sum_{t=1}^{T}|h_{t}^{1}-h_{t}^{D}|=\sum_{t=1}^{T}|\lambda(\omega_{t},\gamma_{t}^{1})-\lambda(\omega_{t},\gamma_{t}^{D})|\leq 
\nonumber
\\
L\sum_{t=1}^{T}\|\gamma_{t}^{1}-\gamma_{t}^{D}\|_{\Gamma}=L\sum_{t=1}^{T}\|\langle \bm{w}_{t}^{1}, \bm{\xi}_{t}\rangle - \langle \bm{w}_{t}^{D}, \bm{\xi}_{t}\rangle \|_{\Gamma}= L\sum_{t=1}^{T}\|\langle \bm{w}_{t}^{1}- \bm{w}_{t}^{D}, \bm{\xi}_{t}\rangle \|_{\Gamma} \leq 
\nonumber
\\
L\sum_{t=1}^{T}\sum_{n=1}^{N}|(w_{t}^{n})^{1}-(w_{t}^{n})^{D}|\cdot \|\xi_{t}^{n}\|_{\Gamma}\leq BL\sum_{t=1}^{T}\sum_{n=1}^{N}|(w_{t}^{n})^{1}-(w_{t}^{n})^{D}|\leq
\nonumber
\\
BLTN\cdot \max_{t, n}|(w_{t}^{n})^{1}-(w_{t}^{n})^{D}|
\label{bound-hard}
\end{eqnarray}
Our goal is to estimate the maximum. In fact, we are to estimate the maximum possible single weight change over $D-1$ steps in algorithm $G_{1}(p)$ or $G_{D}(p)$.

W.l.o.g. we assume that the maximum is achieved at step $t$ on the $1$-st coordinate ($n=1$). We denote $\bm{x}=\bm{w}_{t}^{D}\in \Delta_{N}$ and $\bm{a}\in \Delta_{N}$, where $a_{n}\sim e^{-\eta L_{[t-D+1, t-1]}^{n}}$ (so that $(w_{t}^{n})^{1}\sim a_{n}x_{n}$). The latter equation imposes several restrictions on $\bm{a}$. In fact, for all $n, n'\in\mathcal{N}$ the following must be true: $\frac{a_{n}}{a_{n'}}\leq e^{-\eta (D-1)H}$, where $H$ is the upper bound for the loss $\lambda(\omega, \gamma)$. We denote the subset of such vectors $\bm{a}$ by $\Delta_{N}'\subsetneq \Delta_{N}$.

Note that $\bm{w}_{t}^{1}\sim(x_{1}\cdot a_{1},x_{2}\cdot a_{2},\dots, x_{N}\cdot a_{N})$. We are to bound the maximum

$$\max_{x\in\Delta_{N}}\big[\max_{a\in\Delta_{N}'}\bigg[|x_{1}-\frac{x_{1}a_{1}}{\sum_{n=1}^{N}x_{n}a_{n}}|\bigg]\big].$$

We consider the case $x_{1}>\frac{x_{1}a_{1}}{\sum_{n=1}^{N}x_{n}a_{n}}$ (the other case is similar). In this case

$$\max_{x\in\Delta_{N}}\big[\max_{a\in\Delta_{N}'}\bigg[x_{1}-\frac{x_{1}a_{1}}{\sum_{n=1}^{N}x_{n}a_{n}}\bigg]\big]=\max_{x\in\Delta_{N}}\bigg(x_{1}\big[\max_{a\in\Delta_{N}'}\bigg[1-\frac{a_{1}}{\sum_{n=1}^{N}x_{n}a_{n}}\bigg]\big]\bigg).$$
We examine the behavior of $\big[1-\frac{a_{1}}{\sum_{n=1}^{N}x_{n}a_{n}}\big]$ under the fixed $\bm{x}$.

\begin{equation}\bigg[1-\frac{a_{1}}{\sum_{n=1}^{N}x_{n}a_{n}}\rightarrow \max_{\bm{a}\in \Delta_{N}'}\bigg]\iff \bigg[\frac{\sum_{n=2}^{N}x_{n}a_{n}}{a_{1}}=\sum_{n=2}^{N}\frac{a_{n}}{a_{1}}x_{n}\rightarrow \max_{\bm{a}\in \Delta_{N}'}\bigg]
\label{cond-maximum}\end{equation}
Since $\frac{a_{n}}{a_{1}}\leq e^{-\eta(D-1)H}$, we have $\frac{\sum_{n=1}^{N}x_{n}a_{n}}{a_{1}}\leq (1-x_{1})e^{-\eta(D-1)H}$, and the argument $\bm{a}$ that maximizes \eqref{cond-maximum} does not depend on $x_{2},\dots,x_{N}$. W.l.o.g. we can assume that $\bm{a}=(a, \frac{1-a}{N-1},\frac{1-a}{N-1},\dots,\frac{1-a}{N-1})$, where $a=\frac{e^{-\eta(D-1)H}}{(N-1)+e^{-\eta(D-1)H}}<\frac{1}{N}$. At the same time, we can assume that $\bm{x}=(x,\frac{1-x}{N-1},\frac{1-x}{N-1},\dots,\frac{1-x}{N-1})$.
Thus,
$$\max_{x\in\Delta_{N}}\big[\max_{a\in\Delta_{N}'}\bigg[x_{1}-\frac{x_{1}a_{1}}{\sum_{n=1}^{N}x_{n}a_{n}}\bigg]\big]=\max_{x\in (0, 1)}\big[x-\frac{xa}{xa+\frac{(1-x)(1-a)}{N-1}}\big].$$
The derivative is equal to zero at $x^{*}=\frac{1-a- \sqrt{a(1-a)(N-1)}}{1-aN}$. Substituting $x=x^{*}$ and $a=\frac{e^{-\eta(D-1)H}}{(N-1)+e^{-\eta(D-1)H}}$ we obtain the $N$-independent expression:
\begin{equation}
\label{eqq}
\frac{\big(1-\sqrt{e^{-\eta(D-1)H}}\big)^{2}}{1-e^{-\eta(D-1)H}}.
\end{equation}
We are interested in the behavior near $\eta=0$. A closer look at the Maclaurin series of numerator and denominator give us the decomposition

$$\frac{0+0\eta+\eta^{2}\frac{(D-1)^{2}H^{2}}{4}+\dots}{0+\eta(D-1)H+\dots}.$$
At $\eta\rightarrow 0$ the function is equal to $0$. There exist some $\eta^{0}>0$, such that for all $0\leq \eta<\eta_{0}$ expression \eqref{eqq} can be bounded by some $\eta$-linear function $U(D, H)\cdot \eta = \big[\frac{(D-1)H}{4}+\epsilon\big]\eta$.
Thus, expression \eqref{bound-hard} is bounded by

$$TN\cdot \bigg[B\cdot L\cdot N\cdot U(D, H)\bigg]\cdot \eta=FNT\eta$$
We set $F=B\cdot L\cdot U(D, H)$. Next,
$$H_{T}^{D}\leq H_{T}^{1}+|H_{T}^{1}-H_{T}^{D}|\leq \big[L_{T}^{*}+\frac{\ln N}{\eta}\big]+FNT\eta,$$
where $L_{T}^{*}=\min_{n\in\mathcal{N}}L_{T}^{n}=\min_{n\in\mathcal{N}}\big[\sum_{t=1}^{T}l_{t}^{n}\big]$ is the loss of the best constant expert.

Choosing $\eta^{*}=\argmin_{\eta > 0}\big[\frac{\ln N}{\eta}+FNT\eta\big]=\sqrt{\frac{\ln N}{FNT}}$, we obtain 
$$R_{T}\leq 2\sqrt{FN\ln N}\cdot \sqrt{T}=O(\sqrt{N\ln N}\sqrt{T})$$ regret bound with respect to the best expert in the pool. 

Note that in order to use linear approximation of maximum we need $\eta\leq\eta^{0}$. Moreover, to bound the loss $H_{T}^{1}$ we need $\eta^{*}$-exponentially concave function $\lambda$. This means that $\eta^{*}\leq \min \{\eta^{0}, \eta^{\lambda}\}$. Since ${\eta^{*}=\sqrt{\frac{\ln N}{FNT}}\sim \frac{1}{\sqrt{T}}}$, there exists some huge $T_{0}$ such that for $T\geq T_{0}$ the required conditions are met.\end{proof}

\section{Conclusion}

The problem of long-term forecasting is of high importance. In the article we developed general algorithm $G_{D}(p)$ for the $D$-th-step ahead forecasting with experts' advice (where $p$ is the distribution over active experts). The algorithm uses the ideas of the general aggregating algorithm $G_{1}$ by \cite{Vovk1999}. We also provided the expression for the upper bound for the loss of algorithm $G_{D}(p)$ for any probability distribution $p$ over active experts.

For its important special case $V_{D}^{FC}$ we proved the $O(\sqrt{T})$ regret bound w.r.t. the best expert in the pool. Algorithm $V_{D}^{FC}$ is a practical long-term forecasting modification of algorithm $V_{1}$ by \cite{VoV98}.

It seems possible to apply the approach from the proof of Theorem \ref{theorem-practical-vovk} in order to obtain simpler and more understandable loss bound for algorithm $G_{D}(p)$ for any probability distribution $p$. This statement serves as the challenge for our further research.

\bibliography{references}

\appendix
\section{Optimal Approach to Delayed Feedback Forecasting with Experts' Advice}
\label{optimal-approach}

The appendix is devoted to obtaining the minimax regret bound for the problem of the $D$-th-step-ahead forecasting as a function of the minimax bound for the $1$-step-ahead forecasting. We consider the general protocol of the $D$-th-step-ahead forecasting game  with experts' advice from Subsection \ref{ltf-pea}. Within the framework of the task, we desire to compete with the best expert in the pool. 

We use $\Omega_{t}$ to denote the sequence $(\omega_{1},...,\omega_{t})$ of time series values at first $t$ steps. Pair $I_{t}=(\Omega_{t}, \bm{\Xi}_{t+D})$ is the information that an online algorithm knows on the step $t$ of the game.
Let $S_{D, t}$ be the set of all possible online randomized prediction aggregation algorithms with the forecasting horizon $D$ and game length $t$. Each online algorithm $s\in S_{D, t}$ for given time series $\Omega_{t}$ and experts answers $\bm{\Xi}_{t}$ provides a sequence of distributions $\big(\pi_{1}^{s}(\gamma),...,\pi_{t}^{s}(\gamma)\big)$, where each distribution $\pi_{\tau}^{s}(\gamma)$ is a function of $I_{\tau-D}$ for $\tau=1,\dots,t$. We write $\pi_{\tau}^{s}(\gamma)=\pi^{s}(\gamma|I_{\tau-D})$.

The  expected cumulative loss of the algorithm $s\in S_{D, t}$ on a given $I_{T}$ is

$$H^{D}_{t}(s, I_{t})=\sum_{\tau=1}^{t}\big[\mathbb{E}_{\pi_{\tau}^{s}}\lambda(\omega_{\tau}, \gamma)\big]$$
and the cumulative loss of expert $n$:

$$L_{t}^{n}(I_{t})=\sum_{\tau=1}^{t}\lambda(\omega_{\tau}, \xi_{\tau}^{n})$$

The online performance of the algorithm $s\in S_{D, t}$ for the given $I_{t}$ is measured by the expected cumulative regret over $t$ rounds:

$$R^{D}_{t}(s, I_{t})=H^{D}_{t}(s, I_{t})-\min_{n\in \mathcal{N}}L_{t}^{n}(I_{t})$$
with respect to the best expert. Here $\lambda(\omega, \gamma): \Omega\times\Gamma\rightarrow\mathbb{R}_{+}$ is some loss function, not necessary convex, Lipshitz or exponentially concave. The performance of the strategy $s$ is

$$R_{t}^{D}(s)=\max_{I_{t}}R_{t}^{D}(s, I_{t}),$$
that is the maximal expected regret over all possible games $I_{t}$.
The strategy $s_{D, t}^{*}$ that achieves the minimal regret
$$s_{D,t}^{*}=\argmin\limits_{s_{D}\in S_{D, t}}R_{t}^{D}(s)$$
is called optimal (may not be unique).

\begin{theorem}
For the given $\Omega_{T}, \bm{\Xi}_{T}$, forecasting horizon $D$ and game length $T$ (such that $T$ is a multiple of $D$) we have

$$R_{T}^{D}(s_{D, T})\geq D\cdot R_{T/D}^{1}(s_{1, T/D}).$$
\label{optimality-theorem}
\end{theorem}
\begin{proof}
Let $s=s_{D,T}^{*}\in S_{D, T}$ be the optimal strategy. We define new one-step ahead forecasting strategy $s''\in S_{1, T}$ based on $s$. Let

$$\pi^{s''}_{t}=\frac{1}{D}\sum_{\tau=[\frac{t}{D}]D+1}^{[\frac{t}{D}]D+D}\pi_{\tau}^{s}=\frac{1}{D}\sum_{\tau=[\frac{t}{D}]D+1}^{[\frac{t}{D}]D+D}\pi^{s}(\gamma|I_{\tau-D}).$$
For every one-step ahead forecasting game $I_{T/D}'=(\Omega_{T/D}', \bm{\Xi}_{T/D}')$ we create a new $D$-step ahead forecasting game $I''_{T}=(\Omega_{T}'', \bm{\Xi}_{T}'')$. We set $\omega_{t}''=\omega'_{[t/D]+1}$ and $(\xi_{t}^{n})''=(\xi_{[t/D]+1}^{n})'$ for all $t=1,2,\dots,T$ and $n\in\mathcal{N}$.

The last step is to define one-step ahead forecasting strategy $s'$ for one-step ahead forecasting game $I_{T/D}'$. We set $\pi^{s'}_{t}=\pi_{(t-1)D+1}^{s''}.$

We compute the loss of the algorithm $s$ on $I''_{T/D}$.

\begin{eqnarray}
H^{D}_{T}(s, I_{t}'')=\sum_{t=1}^{T}\mathbb{E}_{\pi_{t}^{s}}\big[\lambda(\omega_{t}'', \gamma)\big]=
\nonumber
\\
\sum_{t=1}^{T}\mathbb{E}_{\pi_{t}^{s}}\big[\lambda(\omega_{[t/D]+1}', \gamma)\big]= 
\sum_{t=1}^{T/D}\bigg[\sum_{\tau=1}^{D}\mathbb{E}_{\pi_{(t-1)D+\tau}^{s}}
\big[\lambda(\omega_{t}'', \gamma)\big]\bigg]=
\nonumber
\\
\sum_{t=1}^{T/D}\bigg[D\cdot \mathbb{E}_{\pi_{(t-1)D+1}^{s''}}\big[\lambda(\omega_{t}'', \gamma)\big]\bigg]=
D\sum_{t=1}^{T/D}\bigg[\mathbb{E}_{\pi_{(t-1)D+1}^{s''}}\big[\lambda(\omega_{t}'', \gamma)\big]\bigg]=
\nonumber
\\
=D\sum_{t=1}^{T/D}\bigg[\mathbb{E}_{\pi_{t}^{s'}}\big[\lambda(\omega_{t}'', \gamma)\big]\bigg]=
D\cdot H^{1}_{T/D}(s', I_{T/D}')
\nonumber
\end{eqnarray}
Also note that 
\begin{equation}
L_{T}^{n}(I_{T}'')=D\cdot L_{T/D}^{n}(I_{T/D}').
\label{loss-transition-eq}
\end{equation}
Thus, for every $I_{T/D}'$ we have 

$$R^{D}_{T}(s, I_{t}'')=D\cdot R^{1}_{T/D}(s', I_{T/D}').$$

According to the definition of the minimax regret, for the one-step forecasting game of length $T/D$ there exists such $I_{T/D}'$ that

$$R^{1}_{T/D}(s', I_{T/D}')\geq R_{T/D}^{1}(s_{1, T/D})$$.

For $I_{T/D}'$ and the corresponding sequence we obtain

$$R^{D}_{T}(s)\geq R^{D}_{T}(s, I_{t}'')=D\cdot R^{1}_{T/D}(s', I_{T/D}')\geq D\cdot R_{T/D}^{1}(s_{1, T/D})$$
This ends the proof.
\end{proof}

In fact, from Theorem \ref{optimality-theorem} we can conclude that an optimal aggregating algorithm $A_{D}^{*}$ for the long-term forecasting with experts' advice and competing with the best expert can be obtained by a simple replicating technique from the optimal algorithm $A_{1}^{*}$ for the $1$-step ahead forecasting and competing with the best expert.

However, if the goal is not to compete with the best expert (but, for example, to compete with the best alternating sequence of experts with no more than K switches or some other limited sequence), this theorem may also work. All the computations still remain true in the general case, except for \eqref{loss-transition-eq}. This equality should be replaced (if possible) by the analogue.

\end{document}